\documentclass[11pt]{article}

\usepackage[truedimen,margin=25truemm]{geometry}

\usepackage{color}

\usepackage{ulem}

\usepackage{booktabs}

\usepackage{hyperref}

\usepackage{graphicx} 
\usepackage{epsfig}
\usepackage{subfigure}
\usepackage{bm}

\usepackage{amssymb}
\usepackage{amsmath}
\usepackage{amsthm}
\usepackage{mathtools}
\usepackage[raggedright]{titlesec} 
\usepackage{stmaryrd}
\usepackage{ulem}

\usepackage{algorithm}
\usepackage{algorithmicx}
\usepackage{algpseudocode}

\newtheorem{theorem}{Theorem}
\newtheorem{definition}{Definition}

\newcommand{\paratitle}[1]{\noindent {\bf #1}}

\def\vec#1{\mbox{\boldmath $#1$}}
\def\mat#1{\mbox{\bf #1}}

\title{Self-supervised Subgraph Neural Network With Deep Reinforcement Walk Exploration}

\author{Jianming Huang \thanks{Graduate School of Fundamental Science and Engineering, WASEDA University, 3-4-1 Okubo, Shinjuku-ku, Tokyo 169-8555, Japan (e-mail: koukenmei@toki.waseda.jp) } \and Hiroyuki Kasai \thanks{Department of Computer Science and Communication Engineering, WASEDA University, 3-4-1 Okubo, Shinjuku-ku, Tokyo 169-8555, Japan (e-mail: hiroyuki.kasai@waseda.jp)}}

\begin{document}

\maketitle

\begin{abstract}
Graph data, with its structurally variable nature, represents complex real-world phenomena like chemical compounds, protein structures, and social networks. Traditional Graph Neural Networks (GNNs) primarily utilize the message-passing mechanism, but their expressive power is limited and their prediction lacks explainability. To address these limitations, researchers have focused on graph substructures. Subgraph neural networks (SGNNs) and GNN explainers have emerged as potential solutions, but each has its limitations. SGNNs computes graph representations based on the bags of subgraphs to enhance the expressive power. However, they often rely on predefined algorithm-based sampling strategies, which is inefficient. GNN explainers adopt data-driven approaches to generate important subgraphs to provide explanation. Nevertheless, their explanation is difficult to be translated into practical improvements on GNNs. To overcome these issues, we propose a novel self-supervised framework that integrates SGNNs with the generation approach of GNN explainers, named the Reinforcement Walk Exploration SGNN (RWE-SGNN). Our approach features a sampling model trained in an explainer fashion, optimizing subgraphs to enhance model performance. To achieve a data-driven sampling approach, unlike traditional subgraph generation approaches, we propose a novel walk exploration process, which efficiently extracts important substructures, simplifying the embedding process and avoiding isomorphism problems. Moreover, we prove that our proposed walk exploration process has equivalent generation capability to the traditional subgraph generation process. Experimental results on various graph datasets validate the effectiveness of our proposed method, demonstrating significant improvements in performance and precision.
\end{abstract}

\section{Introduction}
{Graph data, as a structurally variable data structure, exhibits more dimensions of variability compared to traditional matrix or array data. Complex real-world phenomena, such as chemical compounds, protein structures, and social networks, can be conveniently represented using graph data. However, this flexibility introduces challenges in feature extraction due to the inherent uncertainty in graph structures. Most mainstream Graph Neural Networks (GNNs) \cite{Kipf_arXiv_2017, xu2018powerful, wijesinghe2021new, wang2019dynamic} rely on the message-passing mechanism. These models are namely message passing neural networks (MPNNs). MPNNs learn graph convolution filters by propagating and aggregating features across nodes. Although the MPNN framework can efficiently capture information from node neighborhoods, its expressive power is inherently limited. Specifically, it has been proven to be equivalent to $1$-Weisfeiler-Lehman (WL) test \cite{xu2018powerful}. To address this limitation, researchers have proposed the high-order GNNs \cite{morris2019weisfeiler} that extend the expressive power to $k$-WL. However, implementing high-order GNNs poses challenges due to their complexity and computational cost, limiting their practical use to no more than $3$rd order.} Numerous studies have demonstrated the significance of graph substructures in various graph-related tasks across different domains \cite{prvzulj2007biological, jiang2010finding, alsentzer2020subgraph, sun2021sugar}. However, in the message passing mechanism, akin to a diffusion-like propagation process, the size of the involved neighborhood grows exponentially as iteration number increases. Consequently, this growth leads to a reduction in the prominence of important substructures, rendering it susceptible to interference from other less relevant components within the local substructure. {Therefore, researchers have increasingly directed their attention toward substructures in graph learning. Some studies focus on leveraging substructure information to enchance traditional GNNs, which are commonly referred to as subgraph neural networks (SGNNs) \cite{bevilacqua2021equivariant, bevilacqua2023efficient}. Some focus on interpreting important substructures from GNNs, known as GNN explainers \cite{ying2019gnnexplainer, shan2021reinforcement, lin2022orphicx}.}

\paratitle{Challenges.} {In this paper, we specially discuss about the challenges brought by SGNNs and GNN explainers, which we believe they have complementary advantages. (i) SGNNs initially generate a bag of subgraphs by selectively removing nodes from the original graph. Subsequently, these subgraphs are encoded using the MPNN framework and aggregated to produce final representations. However, many SGNNs rely on predefined sampling strategies (such as random sampling or algorithm-based approaches) rather than adopting a data-driven approach, leading to inefficiency and instability. (ii) MPNNs are often considered black-box models, lacking the ability to offer interpretable insights into important substructure patterns.} Therefore, many related works to the GNN explainers \cite{ying2019gnnexplainer, shan2021reinforcement, lin2022orphicx} examined the GNN explaination by extracting key subgraph patterns based on trained GNN models. However, despite these efforts, the interpretation of these explanations still requires prior domain knowledge. Furthermore, how to translate this analysis into practical improvements within the GNN framework remains a challenge.

\begin{figure}
\centering
\includegraphics[width= 0.5\hsize]{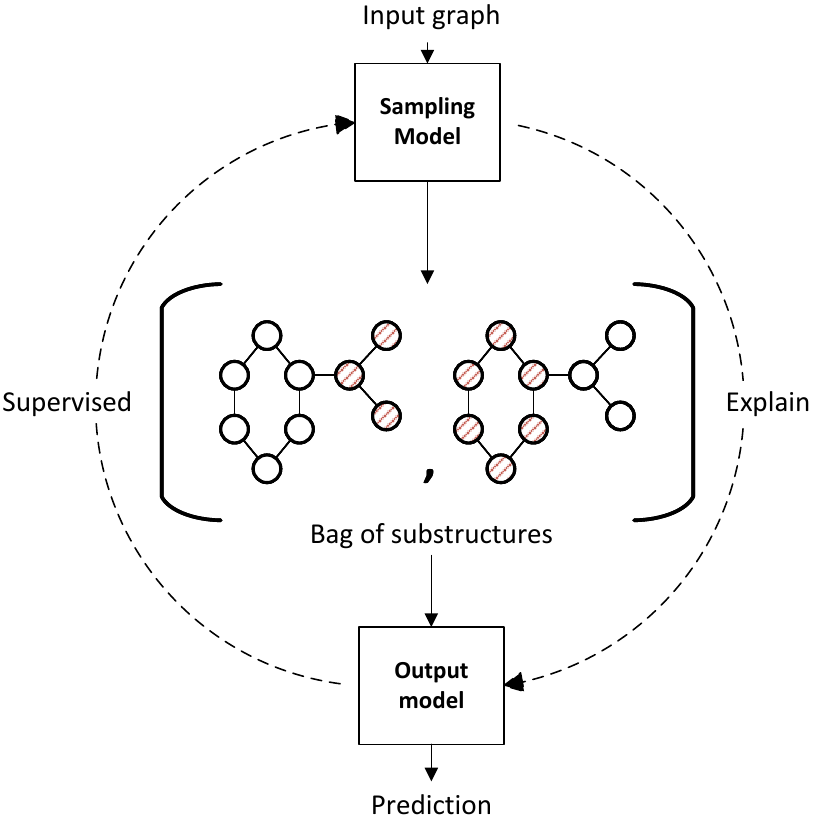}
\caption{Illustration of the self-supervised framework, {where solid arrow represents the direction of data flow and dashed arrow represents the relationship of two models}. The sampling model creates bag of substructures (nodes with red lines denote the substructure) for the output model, which can be considered an explainer of the output model. Conversely, the output model computes predictions, which is used to supervise the training of the sampling model.}
\label{Fig.motiv}
\end{figure}

\paratitle{Motivations.} {In response to the challenges outlined above, we propose to combine the principles of SGNNs with the generation approach employed by GNN explainers to form a novel self-supervised framework. Specifically, as illustrated in Figure \ref{Fig.motiv}, the subgraph sampling model in SGNN is trained using the explainer fashion. In this framework, the sampling model can be considered an explainer of the output model. Its primary function is to optimize the subgraph in a way that enhances the performance of the output model. This optimization process is guided by the principle of maximizing the output model’s performance, where the output model plays the role of a supervisor. By doing so, we can address the aforementioned challenges effectively.} 
Subsequently, the question arises: how should we design the {sampling model}? Fortunately, recent research on explainers for GNNs offers valuable insights. These explainers can effectively extract the graph topology of important substructures. Recent works \cite{shan2021reinforcement, yuan2021explainability, lin2022orphicx} on these explainers are closely related to the deep subgraph generation, which can be divided into two categories: one-shot generation and sequential generation \cite{zhu2022survey}. Within these two categories, we specifically focus on the sequential generation based on reinforcement learning for our proposed framework. Our rationale is as follows: (i) Local relevance: {One-shot generation usually involves computing the correlation between all pairs of nodes \cite{kipf2016variational}. However, this is unnecessary for many graph mining tasks, as most important substructures tend to occupy only a small portion of the entire graph.} Therefore, a localized approach that targets specific subgraph regions is more efficient. (ii) Connectivity assurance: By employing the reinforcement learning strategy, we can easily ensure the connectivity of the generated substructures by adopting node neighbor candidates. This approach helps avoid sparse and scattered results, leading to more coherent and meaningful subgraphs.
{Nevertheless, we still face another challenge: which approach we should choose for the substructure generation process? The old fashion is a subgraph-based generation method, which iteratively select nodes from the subgraph neighborhood. This approach has high computational cost because subgraph neighborhood usually grows exponentially. Conversely, we introduce a walk exploration process. In the walk exploration process, we optimize a path from the entire set of possible random walks generated over the input graph, which can efficiently extract the important substructures. There are two major reasons for this change: (i) Incorporating walk sequences to represent substructures provides a more convenient approach for the embedding process compared to using subgraphs. This convenience arises from the fact that encoding sequences is simpler than encoding graphs, and it eliminates the need to address the isomorphism problem. (ii) Unlike the node-by-node generation process that selects candidates from subgraph boundaries, our walk exploration process considers candidates from neighbors of the current node. This approach reduces the number of candidates and improves efficiency.}

\begin{figure*}
\centering
\includegraphics[width= \hsize]{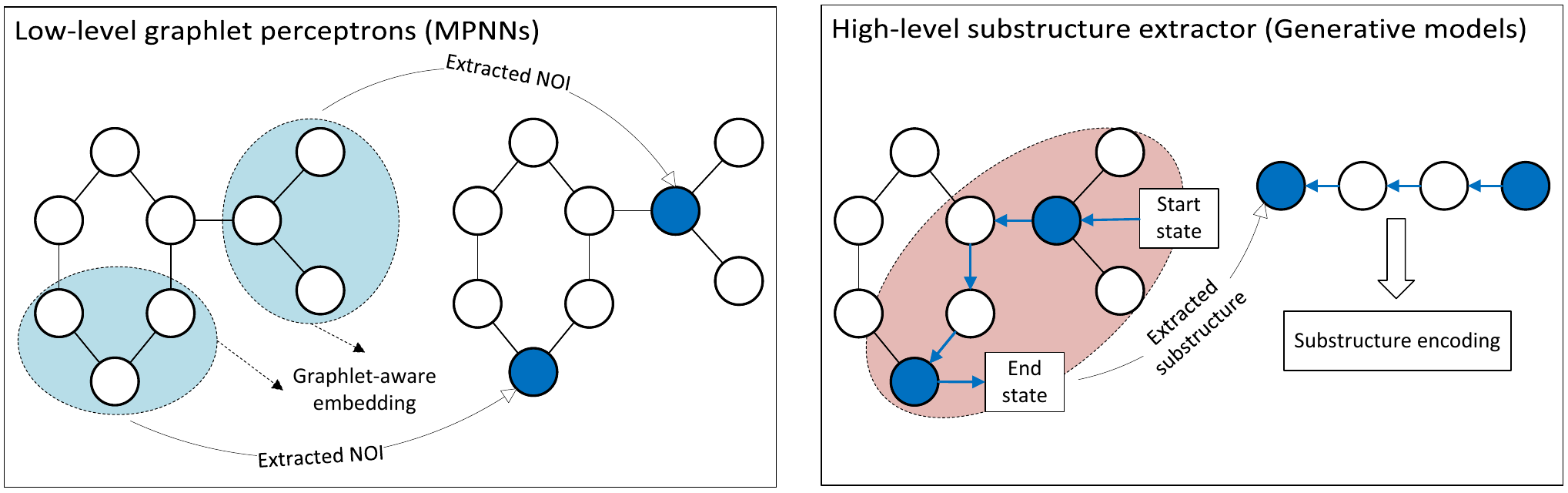}
\caption{Illustration of the sampling model of RWE-SGNN, which shows the process of extracting the important substructure from a graph. The red region highlights the important substructure area. Initially, the low-level graphlet perceptrons detect the key graphlets and assign graphlet-aware embeddings to all the nodes. Subsequently, with these graphlet-aware embeddings, the high-level reinforcement walk exploration framework extracts node of interest and generates a walk sequence specific to the important substructure area. Finally, this walk sequence is fed into a sequence encoder to compute embeddings for the bags of substructures.}
\label{Fig.overall}
\end{figure*}

\paratitle{Present work.} In this work, we propose a self-supervised SGNN based on reinforcement walk exploration (RWE-SGNN). {In Figure \ref{Fig.overall}, we first introduce the sampling model of our proposed method, which consists of two key components: a low-level graphlet perceptron and a high-level substructure extractor. The primary objective of the low-level graphlet perceptron is to identify nodes of interest (NOIs) that exhibit expected graphlet patterns within their local neighborhoods. This functionality is achieved using a shallow MPNN. On the other hand, the high-level substructure extractor leverages a reinforcement generative model to generate important substructures from the original graph based on the identified NOIs. We define a walk-exploration-based Markov decision process (MDP) for this reinforcement generative model.} In the subsequent steps, the extracted bag of walk sequences are fed into a sequence encoder, which computes subgraph embeddings for downstream tasks. {Additionally, in order to achieve the self-supervised function in Figure \ref{Fig.motiv},} we design a two-stage training process. Our two-stage training process {treats the sampling model and the output model as two sub-models and trains them alternately}. The output model is responsible for encoding and decoding the extracted bag of walk sequences to produce the desired outputs for downstream tasks. During training, we optimize the output model using downstream losses while keeping the walk exploration agent’s parameters fixed. Subsequently, we leverage the downstream losses, combined with the trained and fixed output model, as the reward function for reinforcement learning of the walk exploration agent.

Our contributions can be summarized as follows: {(i) we propose a self-supervised framework by combing the idea of SGNNs and GNN explainers. The proposed framework resolves the difficulties of inefficient and unstable sampling policy in traditional SGNNs. Furthermore, it also has good interpretability by providing visble explanation of important substructures from graph data. (ii) we propose a novel walk-exploration-based MDP for our sampling model based on reinforcement learning. It greatly reduces the computational complexity compared to the traditional subgraph-based generation approach. Furthermore, we also prove that our proposed walk-exploration-based MDP has equivalent substructure generation capability to the traditional subgraph-based generation approach.} (iii) we conduct experiments over several widely used graph datasets, which demonstrates the effectiveness of our proposed method.

\section{Preliminaries}
Herein, we use lowercase letters to represent scalars such as $a,b$, and $c$. Lowercase letters with bold typeface such as $\vec{a}, \vec{b},$ {and} $\vec{c}$ represent vectors. Uppercase letters with bold typeface such as $\mat{A}, \mat{B},$ {and} $\mat{C}$ represent the matrices. For slicing matrices, we represent the element of the $i$-th row and $j$-th columns of $\mat{A}$ as $\mat{A}(i,j)$. Furthermore, $\mat{A}(i,:)$ denotes the vector of the $i$-th row of $\mat{A}$. Similarly, $\mat{A}(:,j)$ signifies the vector of $j$-th column. Also, $\vec{a}(i)$ denotes the $i$-th element of the vector $\vec{a}$.
Uppercase letters with an italic typeface denote a set, such as $\mathcal{A},\mathcal{B}, \mathcal{C}$. Also, $|\mathcal{A}|$ denotes the size of $\mathcal{A}$. $\mathbb{R}$ stands for the real number set.

\paratitle{Graph}. A graph can be denoted as $G:=(\mathcal{V},\mathcal{E})$ in discrete mathematics, which is a tuple of a node set $\mathcal{V}$ and an edge set $\mathcal{E}$. More specifically, let $\mathcal{V}:=\{v:v\in\mathbb{N}^+\}$ be a node index set and let $\mathcal{E}:=\{(v, v^\prime):v,v^\prime\in\mathcal{V}\}$ be a set of node pairs. It is noteworthy that we only discuss about the undirected graphs, therefore, $(v,v^\prime)$ and $(v^\prime, v)$ are considered equivalent in $\mathcal{E}$. For a node $v\in\mathcal{V}$, we write $\mathcal{N}(v):=\{v^\prime:v^\prime\in\mathcal{V};(v,v^\prime)\in\mathcal{E};v\neq v^\prime\}$ as the set of node neighbors. {In graph data, each node is assigned with a feature vector. We use the node feature matrix $\mat{X}\in\mathbb{R}^{n\times d}$ to represent these features, where $n$ and $d$ denote the number of nodes and feature dimensions, respectively.} 

\paratitle{SGNNs}. {SGNNs, also known as subgraph GNNs, compute graph representations based on bags of sampled subgraphs. These models typically consist of four components: (i) a subgraph sampling policy extracting subgraphs from the original graph, (ii) permutation-equivariant layers that convert bags of subgraphs into bags of embeddings while considering their natural symmetry (usually implemented using MPNNs), (iii) a pooling function to aggregate the bags of embeddings, and (iv) a multi-layer perceptron (MLP) network for final outputs.}

\paratitle{Post-hoc GNN explainers}. {A post-hoc GNN explainer aims to identify a subgraph denoted as $G^\prime$ from an input graph $G$. This subgraph is considered important subgraph because it can explain the predictions made by a pretrained GNN model $f$ on $G$. Specifically, the effectiveness of the explanation is often evaluated based on the prediction consistency of $f$ between $G^\prime$ and the original graph $G$ \cite{ying2019gnnexplainer, shan2021reinforcement, lin2022orphicx}.}

\section{Proposal}

\begin{figure*}
\centering
\includegraphics[width= \hsize]{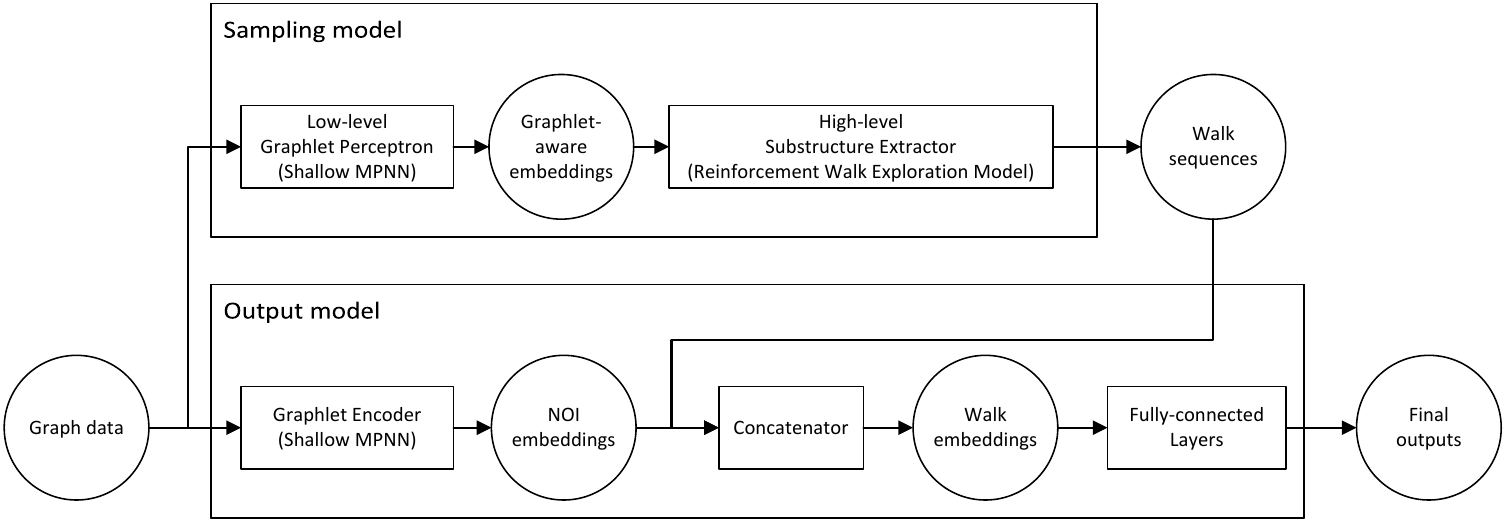}
\caption{Illustration of the model architecture, which comprises two main components: the sampling model and the output model. Initially, the input graph data are processed by the sampling model, which generates important substructures in a data-driven manner. Subsequently, leveraging the generated walk sequence, the output model encodes the extracted substructures for downstream tasks.}
\label{Fig.model_arch}
\end{figure*}

This section introduces the details of our proposed RWE-SGNN in several aspects: (i) Model architecture: we present the architectural details of RWE-SGNN, highlighting its design and components; (ii) Optimization problem for structure generation: we define the optimization problem that involves finding an optimal walk to represent important substructures. (iii) Markov decision process (MDPs): we introduce the concepts of MDP in the context of our reinforcement walk learning process {and compare it with a subgraph-based generation process}. (iv) Two-stage training process: we introduce the two-stage reinforcement learning process, which optimizes a sampling model and an output model. (v) Complexity: we analyze the computational complexity of RWE-SGNN.

\subsection{Model architecture}

In the introduction section, we present our proposal for a self-supervised SGNN capable of dynamically extracting important substructures from graph data. To enhance the generalizability of our framework across various downstream tasks, we decompose it into two sub-models: the sampling model and the output model. As illustrated in Figure \ref{Fig.model_arch}, the sampling model is responsible for {the data-driven substructure sampling (generation),} while the output model is responsible for the substructure encoding ensures that the outputs align with the requirements of specific downstream tasks. Both of these sub-models incorporate a shallow MPNN that computes node-level embeddings. However, their roles differ. In the sampling model, the MPNN calculates low-level graphlet-aware embeddings to guide the substructure generation process within the high-level reinforcement learning model. Conversely, in the output model, the MPNN encodes the extracted NOIs and aggregates them into substructure representations. These subgraph representations are then fed into fully-connected layers (MLP) to produce the final output. Building upon this architechture, our proposed RWE-SGNN solves the following subgraph optimization problem.

\begin{definition}[Subgraph optimization problem]
\label{Def.Sop}
Consider a graph dataset denoted as $\mathcal{G}$, along with a downstream task loss function $l: \mathbb{R}^d \to \mathbb{R}$ that operates on $d$-dimensional embeddings. The RWE-SGNN learns two key components: a sampling model $S_{\theta}:\mathcal{G}\to\widehat{\mathcal{G}}$ parameterized by $\theta$, and an output model $O_{\theta^\prime}:\widehat{\mathcal{G}}\to\mathbb{R}^d$ parameterized by $\theta^\prime$, where $\widehat{\mathcal{G}}:=\{G^\prime(\mathcal{V}^\prime,\mathcal{E}^\prime):\exists G(\mathcal{V},\mathcal{E})\in\mathcal{G},\mathcal{V}^\prime\subseteq\mathcal{V}, \mathcal{E}^\prime\subseteq\mathcal{E}\}$ denotes an induced subgraph set derived from graphs of $\mathcal{G}$. The goal is to optimize the following equation.
\begin{equation}
	\mathop{\rm minimize}_{\theta,\theta^\prime} \sum_{g\in\mathcal{G}}l(O_{\theta^\prime}(S_{\theta}(g))).\label{Eq.Sop}
\end{equation}
\end{definition}
The learning process for solving this problem diverges from the typical deep learning process, because we incorporate a reinforcement learning subgraph generator within the sampling model. This subgraph generator involves several discrete operations, including node indication and subgraph border neighborhood extraction. Specifically, we propose a two-stage optimization process. First, we define a {dynamic reward function based on $l(O_{\theta^\prime}(\hat{g}))$ for the sampling model $S_{\theta}$, where $\hat{g}$ denotes a generated substructure}. Next, we optimize the parameters $\theta$ and $\theta^\prime$ independently and alternately, while keeping the other fixed.

\subsection{Definitions of substructure generation MDP}
\label{Sec.MDPs}
In our approach, akin to many subgraph extraction methods based on reinforcement learning, we formulate the structure generation process as an MDP. The MDP is represented as a 4-tuple $\mathcal{M} := (\mathcal{S}, \mathcal{A}, P, R)$, where $\mathcal{S}$ denotes the state space, $\mathcal{A}$ denotes the action space, $P: \mathcal{S} \times \mathcal{A} \times \mathcal{S} \to [0, 1] \cap \mathbb{R}$ returns the probability of a 3-tuple transition $(s, a, s^\prime)$ from state $s \in \mathcal{S}$ to state $s^\prime \in \mathcal{S}$ due to action $a \in \mathcal{A}$, and $R: \mathcal{S} \times \mathcal{A} \times \mathcal{S} \to \mathbb{R}$ represents the immediate reward function for feasible transitions. Notably, although our structure generation process is an unconditional generation process, the available candidate set for each generation step is finite and typically obtained from a predefined subgraph neighborhood.

While many existing works on GNN explainers offer ready-made frameworks for substructure generation, integrating these frameworks with GNNs for graph mining tasks remains challenging due to their high complexity. To address this, we introduce a reinforcement walk exploration process as an alternative to the conventional subgraph generation approach. Our proposed method significantly reduces the size of the feasible action space at each generation step and eliminates the need for explicit subgraph encoding. To assess the efficiency of our approach compared to the traditional method, we also introduce a variant based on our model architecture that incorporates the conventional subgraph generation process. {We subsequently compare our proposed method with this variant through complexity analysis and numerical experiments to demonstrate the efficiency and effectiveness of our walk exploration approach.} 

\paratitle{Subgraph-generation-based MDP}. In the context of reinforcement subgraph generation, existing works \cite{shan2021reinforcement, yuan2021explainability, trivedi2020graphopt} are valuable references. However, due to the distinction between graph mining tasks and GNN explanation analysis tasks, our proposed method prioritizes high efficiency and low complexity. Consequently, our framework primarily focuses on reducing the size of the available state space and action space, even if it means sacrificing some interpretability within acceptable bounds. It’s important to note that interpretability, while an auxiliary property, is not the primary objective in graph mining tasks. As a result, we contend that sacrificing a degree of accuracy in the extracted subgraph for the sake of explainability will not significantly impact human understanding of key subgraphs. Consequently, we narrow down the state space from the induced subgraph set to the connected subgraph set. Additionally, we limit the action space from the entire set of unselected nodes to the neighborhood of border nodes. This decision is motivated by the fact that subgraphs containing isolated components tend to be difficult to interpret, and aggregating information from these isolated parts is a challenge. We then define the subgraph-generation-based MDP as follows.

\begin{definition}[Subgraph-generation-based MDP]
	\label{Def.subg_mdp}
	Given an input graph $G(\mathcal{V},\mathcal{E})\in\mathcal{G}$, the subgraph-generation-based MDP $\mathcal{M}_g := (\mathcal{S}_g, \mathcal{A}_g, P_g, R_g)$ consists of the following four concepts.
	\begin{itemize}
		\item \textbf{State}. The state space w.r.t. the input graph $G(\mathcal{V},\mathcal{E})$ is defined as its connected subgraph set $\mathcal{S}_g:=\{G^\prime(\mathcal{V}^\prime,\mathcal{E}^\prime):\mathcal{V}^\prime\subseteq\mathcal{V}, \mathcal{E}^\prime\subseteq\mathcal{E}, \forall u,v\in\mathcal{V}^\prime, u, v {\rm\ are\ connected.}\}$. Specifically, a state is defined as a subgraph structure derived from $G(\mathcal{V},\mathcal{E})$ and is denoted as $s:=G_s(\mathcal{V}_s,\mathcal{E}_s)$. 
		\item \textbf{Action}. The complete action space is defined as the node set of input graph $\mathcal{A}_g := \mathcal{V}$, where an action is defined as a neighbor node $a := v_a$. Furthermore, the feasible action space for a state $s \in\mathcal{S}_g$ is defined as the union of its border neighborhood and is denoted as $\widehat{\mathcal{A}}_g(s) := \bigcup_{v\in\mathcal{V}_s}\mathcal{N}(v) \setminus \mathcal{V}_s$.
		\item \textbf{Transition}. The transition between different states is defined as a node adding operation, which involves the addition of a neighbor node to the current state $s$ and also the addition of its connection edges to other nodes already existing in $s$. The transition operator from the current state $s$ is denoted as $T_s^g:\widehat{\mathcal{A}}_g(s)\to \mathcal{S}_g$, where $T_s^g(a) = s^\prime, \mathcal{V}_{s^\prime} = \mathcal{V}_s\cup\{v_a\}, \mathcal{E}_{s^\prime} = \mathcal{E}_s\cup\{(v_a, j):j\in\mathcal{N}(v_a)\cap\mathcal{V}_s\}$.
		\item \textbf{Reward}. The immediate reward function $R_g: \mathcal{S}_g \times \mathcal{A}_g \times \mathcal{S}_g \to \mathbb{R}$ is defined based on the downstream tasks. Given an output model $O_{\theta^\prime}$ and a loss function $l$ presented in Equation (\ref{Eq.Sop}), we define the reward function as $R_g(s,a,s^\prime):=l(O_{\theta^\prime}(s)) - l(O_{\theta^\prime}(s^\prime)), s\in\mathcal{S}_g, s^\prime = T_s^g(a)$.
	\end{itemize}
\end{definition}

\paratitle{Walk-exploration-based MDP}. Despite the reduction from induced subgraph to connected subgraph in our proposed subgraph-generation-based MDP, the action space remains substantial. This is because the subgraph generation is actually a breadth-first-search process that selects candidates from all the surrounding neighbors of the generated subgraph. The size of the feasible action space can be calculated as $|\widehat{\mathcal{A}}_g(s)| = |\bigcup_{v\in\mathcal{V}_s} \mathcal{N}(v)\setminus\mathcal{V}_s|$ according to {\bf Definition \ref{Def.subg_mdp}}. The equation reveals a potential risk of candidate explosion as the subgraph size $|\mathcal{V}_s|$ grows larger. To mitigate this risk and further reduce the action space, we propose a depth-first-search strategy based on walk exploration. In walk exploration, we utilize a mechanism akin to the random walk process. Specifically, the next step of a walk is chosen from the neighborhood of the current node. Reinforcement learning plays an important role in guiding the direction of the walk. We define the Walk-exploration-based MDP as follows.

\begin{definition}[Walk-exploration-based MDP]
	\label{Def.walk_mdp}
	Given an input graph $G(\mathcal{V},\mathcal{E})\in\mathcal{G}$ and the maximum length of walks $L\in\mathbb{N}^+$, the Walk-exploration-based MDP $\mathcal{M}_w := (\mathcal{S}_w, \mathcal{A}_w, P_w, R_w)$ consists of the following four concepts.
	\begin{itemize}
		\item \textbf{State}. The state space w.r.t. the input graph $G(\mathcal{V},\mathcal{E})$ is defined as the random walk set on $G$ of length less than $L+1$ and is denoted as $\mathcal{S}_w:=\{(v_1,v_2,\cdots,v_l):\forall i\in\llbracket l-1\rrbracket, l\leq L, v_i,v_l\in\mathcal{V}, (v_i, v_{i+1})\in\mathcal{E}\}$. Specifically, a state is defined as a walk sequence and is denoted as $s:=(v_1^s,v_2^s,\cdots,v_l^s)$. 
		\item \textbf{Action}. {We use the similar definition to the $\mathcal{A}_g$ in {\bf Defintion \ref{Def.subg_mdp}} to define the walk action space $\mathcal{A}_w:=\mathcal{V}$.} Furthermore, the feasible action space for a state $s \in\mathcal{S}_w$ is defined as the node neighborhood of the current step and is denoted as $\widehat{\mathcal{A}}_w(s) := \mathcal{N}(v^s_l)$.
		\item \textbf{Transition}. The transition between different states is defined as a node appending operation, which involves the appending operation of a neighbor node to the walk sequence of $s$. The transition operator from the current state $s$ is denoted as $T^w_s:\widehat{\mathcal{A}}_w(s)\to \mathcal{S}_w$, where $T^w_s(a) = s\oplus v_a$ and $\oplus$ denotes the appending operator to a walk sequence.
		\item \textbf{Reward}. {Similar to the reward definition $R_g$ in {\bf Defintion \ref{Def.subg_mdp}}, we define the walk-exploration-based} reward function as $R_w(s,a,s^\prime):=l(O_{\theta^\prime}(s)) - l(O_{\theta^\prime}(s^\prime)), s\in\mathcal{S}_w, s^\prime = T^w_s(a)$.
	\end{itemize}
\end{definition}

{Importantly, we do not employ an additional pretraining step for starting point selection. In both subgraph-generation-based and walk-exploration-based MDPs, the process begins with an empty substructure as the initial state, where all nodes in the graph are considered candidates. Moreover, the generation process terminates either when the substructure reaches its maximum size or when no available candidates remain. To assess the generation capability of our proposed walk exploration approach, we propose the following theorem.}

\begin{theorem}
\label{Theom.1}
Given a graph $G(\mathcal{V},\mathcal{E})$. For any connected subgraph $G^\prime(\mathcal{V}^\prime,\mathcal{E}^\prime)$ of $G$, let $\mathcal{W}_{G^\prime}$ denote the complete random walk set on $G^\prime$, which includes all possible random walk sequences of arbitrary length. There always exists at least one random walk sequence in $\mathcal{W}_{G^\prime}$ that can visit all nodes in $\mathcal{V}^\prime$.
\end{theorem}

\begin{proof}
{Let $a := (v_1, v_2, \ldots, v_n)$ be a sequence containing all nodes in $\mathcal{V}^\prime$, arranged in an arbitrary order, where $n = |\mathcal{V}^\prime|$. For any consecutive pairs $(v_t, v_{t+1})$ in $a$, there always exists a walk sequence $w_{t,t+1} \in \mathcal{W}_{G^\prime}$ that starts from $v_t$ and ends at $v_{t+1}$ because $G^\prime$ is a connected subgraph. Consequently, we can construct a walk containing all nodes in $\mathcal{V}^\prime$ by concatenating $w_{1,2}, w_{2,3}, \ldots, w_{n-1,n}$ into a single walk sequence.}
\end{proof}

{\bf Theorem \ref{Theom.1}} establishes that the subgraph generation capability achieved through walk exploration is equivalent to our proposed subgraph generation strategy. By replacing the subgraph generation process with walk exploration, we significantly reduce the size of the feasible action space to $|\mathcal{N}(v^s_l)|$ while keeping the same capability to subgraph generation. Although the depth-first-search method may lead to an increased number of necessary generation steps due to backtracing exploration, this trade-off remains acceptable compared to the candidate explosion. We found that the primary scenario for backtracing occurs in star-shaped subgraphs, where walks backtrack from the end of one branch to explore other branches. Fortunately, this challenge is mitigated by our adoption of a shallow MPNN as the graphlet encoder as shown in Figure \ref{Fig.model_arch}, which avoids exhaustive searches of small branches.

\subsection{Sampling model with reinforcement learning}
\label{Sec.SGM}

Overall the sampling model is divided into two parts: a low-level graphlet perceptron and a high-level substructure generator. We introduce the implementation details about the each component separately in this subsection. For the detailed pseudo-code algorithms of the training and inference process, please refer to \textbf{Algorithm \ref{Alg.train}} and \textbf{Algorithm \ref{Alg.test}} in the appendix.

\paratitle{Low-level graphlet perceptron}. In accordance with the introduction section, we implemente a low-level graphlet perceptron using a shallow MPNN to detect simple graphlets. Let ${\rm MPNN}_{\Theta}^{(p,d,k)}$ denote a $p$-layer MPNN parameterized by $\Theta$, with an input channel size of $d$ and an output channel size of $k$. Given a graph input $G(\mathcal{V},\mathcal{E})$ with a node feature matrix $\mat{X}\in\mathbb{R}^{n\times d}$ (where $|\mathcal{V}| = n$), and each row of $\mat{X}$ representing a $d$-dimensional node feature vector, we compute the graphlet-aware embeddings using the following equation.
\begin{equation}
	\mat{Z} \leftarrow {\rm MPNN}_{\Theta}^{(p,d,k)}(G, \mat{X}),\label{Eq.Zs}
\end{equation}
where $\mat{Z}\in\mathbb{R}^{n\times k}$ denotes the matrix of $k$-dimensional graphlet-aware node embeddings produced by MPNN. These embeddings enhance the reinforcement agent’s awareness of graphlet patterns within node neighborhoods in the following high-level substructure generator.

\paratitle{High-level substructure generator}. The high-level substructure generator is implemented using a reinforcement learning framework based on our proposed MDPs described in Section \ref{Sec.MDPs}. To train an agent model for substructure generation, we employ deep $Q$-networks (DQN) \cite{mnih2013playing}, which is based on the well-known off-policy reinforcement learning algorithm: $Q$-learning \cite{watkins1992q}. We begin by defining the $Q$-function as $Q:\mathcal{S}\times\mathcal{A}\to\mathbb{R}$ according to the Bellman equation, which provides a recursive definition as follows.
\begin{equation}
	Q^\star(s_t, a) := R(s_t, a, s_{t+1}) + \gamma\cdot\mathop{\rm maximize}_{a^\prime}Q^\star(s_{t+1}, a^\prime),\label{Eq.QL}
\end{equation}
where $s_t\in\mathcal{S}$ denotes the current state, and $s_{t+1} := T_{s_t}(a) \in \mathcal{S}$ denotes the next state transited from $s_t$ with action $a$, and $\gamma\in[0,1]\cap\mathbb{R}$ denotes the weight of future rewards. We then define a $\epsilon$-greedy policy function $\pi_\epsilon:\mathcal{S}\times\mathbb{R}\to\mathcal{A}$ using the following equation to sample the generation trajectory sequences, where $\epsilon\in[0,1]\cap\mathbb{R}$ and $q\in\mathbb{R}$ denotes a random variable, and ${\rm Random}(\widehat{\mathcal{A}}(s))$ denotes the operation to randomly select an action from $\widehat{\mathcal{A}}(s)$. It is noteworthy that $q$ is generated from a uniform probability distribution during the training process, while it is always set to be infinite during the inference process.
\begin{equation*}
	\pi_\epsilon(s, q):=\left\{
	\begin{aligned}
	{\rm Random}(\widehat{\mathcal{A}}(s)),&\quad q < \epsilon\\
	\mathop{\rm arg~max}_{a\in\widehat{\mathcal{A}}(s)}Q^\star(s,a),&\quad q\geq\epsilon
	\end{aligned}
	\right.
\end{equation*}
The architecture of the DQN in our sub-model consists of an MLP designed to predict $Q$-values based on the state environments. Given that we propose two categories of MDPs: $\mathcal{M}_g$ and $\mathcal{M}_w$, we define separate state environment encoders for each MDP. First, for the subgraph-generation-based MDP $\mathcal{M}_g$, we introduce the encoder $E_g:\mathcal{S}_g\times\mathbb{R}^{n\times k}\to\mathbb{R}^k$ using Equation (\ref{Eq.subg_enc}), which can be viewed as a global mean pooling on the extracted subgraph with node embeddings. Second, for the walk-exploration-based MDP $\mathcal{M}_w$, we define the state encoder $E_w:\mathcal{S}_w\times\mathbb{R}^{n\times k}\to\mathbb{R}^{Lk}$ based on Equation (\ref{Eq.walk_enc}), which outputs a zero-padded walk sequence vector.
\begin{eqnarray}
	E_g(s, \mat{Z})&:=&\frac{1}{|\mathcal{V}_s|}\sum_{v\in\mathcal{V}_s}\mat{Z}(v,:),\label{Eq.subg_enc}\\
	E_w(s, \mat{Z})&:=&{\rm ZeroPad}_{Ld}\left(\bigoplus_{i=1}^l\mat{Z}(v^s_i,:)\right),\label{Eq.walk_enc}
\end{eqnarray}
where $\mat{Z}$ represents the graphlet-aware embeddings in Equation (\ref{Eq.Zs}), $\bigoplus$ denotes the vector concatenation operator, and ${\rm ZeroPad}_{Ld}$ corresponds to the zero-padding operator that pads the input vector with zeros until its size reaches $L \cdot d$. For the training process of DQN, we adopt two different agent models: policy model and target model, each of them is a parameterized $Q$-function. We denote them as $Q^p_{\Phi}$ and $Q^t_{\Phi^\prime}$, respectively, where $\Phi$ and $\Phi^\prime$ denote their corresponding parameters {(such as MLP parameters)}. Let $\mathcal{T}$ denote the set of $T$-length generation trajectory sequences sampled by $\pi_\epsilon$ at one epoch. Let $\tau:=((s_1,a_1),(s_2, a_2),\cdots,(s_T, a_T))\in\mathcal{T}$ denote one generation trajectory sequence of $\mathcal{T}$, which contains tuples of states and actions of different steps. We optimize the parameters of these two agent models using the following equations.
\begin{eqnarray*}
	Q_E(s,a) &:=& R(s, a, s^\prime) + \gamma\cdot\mathop{\rm maximize}_{a^\prime}Q^t_{\Phi^\prime}(E(s^\prime, \mat{Z}), a^\prime),\\
	\Theta, \Phi &\leftarrow& \mathop{\rm arg~min}_{\Theta, \Phi}\sum_{\tau\in\mathcal{T}}\sum_{(s,a)\in\tau}\left|Q^p_{\Phi}(E(s, \mat{Z}), a) - Q_E(s, a) \right|,\\
	\Phi^\prime &\leftarrow& \beta\cdot\Phi + (1-\beta)\cdot\Phi^\prime,
\end{eqnarray*}
where $\beta\in[0,1]\cap\mathbb{R}$ denote the smoothing factor and $Q_E$ denotes the expected $Q$-value derived from Equation (\ref{Eq.QL}). $s^\prime := T_s(a)$ denotes the state that results from executing action $a$. The MDP components $E$ and $R$ and $T_s$ in the equation can be replaced with the ones corresponding to the subgraph-generation-based MDP or the walk-exploration-based MDP. {It should be noted that, to generate a bag of substructures containing multiple samples, we choose the top-$K$ candidates based on their $Q$-values from the initial state to generate $K$ distinct samples for the same graph.}

\subsection{Output model for downstream task}

In Section \ref{Sec.SGM}, we introduce the sampling model that generates bags of substructures. Subsequently, we introduce an output model designed to transform these generated substructures into the expected outputs, driven by the losses relevant to the downstream tasks. This output model is built upon an encoder-decoder architecture. The encoder of the output model use the same encoder functions as the ones of sampling model in Equations (\ref{Eq.subg_enc}) and (\ref{Eq.walk_enc}). The only difference is that the matrix $\mathbf{Z}$ within these equations is outputed by another shallow MPNN denoted as ${\rm MPNN}_{\Theta^\prime}^{(p,d,k)}$, which is parameterized by $\Theta^\prime$. The rationale behind segregating the MPNNs of the structure generation and output models is to enhance the stability of the learning process. This is attributed to the two-stage updating mechanism for the parameters of these two constituent sub-models. For the decoding phase, we employ an MLP, parameterized by $\Theta^{\prime\prime}$, which operates on the embedded representations of the substructures. The optimization of our output model is based on the subsequent equations when given a graph input $G(\mathcal{V},\mathcal{E})$ with a node feature matrix $\mat{X}\in\mathbb{R}^{n\times d}$.
\begin{eqnarray}
	\mat{Z}^\prime &\leftarrow& {\rm MPNN}_{\Theta^\prime}^{(p,d,k)}(G, \mat{X}),\\
	\Theta^\prime, \Theta^{\prime\prime} &\leftarrow& \mathop{\rm arg~min}_{\Theta^\prime, \Theta^{\prime\prime}} l({\rm MLP}_{\Theta^{\prime\prime}}(f_{\rm PL}(E(S_\theta(G), \mat{Z}^\prime)))),\label{Eq.output}
\end{eqnarray}
where $S_\theta$ denotes the sampling model and $l$ denotes the downstream loss function as described in {\bf Definition \ref{Def.Sop}}. ${\rm MLP}_{\Theta^{\prime\prime}}$ denotes the MLP decoder. $f_{\rm PL}$ denotes a pooling strategy when needed, such as the global mean pooling and the global max pooling.

\subsection{Complexity analysis}

{In this section, we analyze the computational complexity of our proposed walk-exploration-based method compared to the subgraph-generation-based method. Because these two approaches share similar MDP definitions and are implemented within the same framework, we focus on the complexity of their action space. Let $N \in \mathbb{N}^+$ denote the maximum number of nodes in the subgraph-generation-based method, and let $L \in \mathbb{N}^+$ denote the maximum length of walks in the walk-exploration-based method. 
For the subgraph-generation-based method, the number of processed candidates required to generate a subgraph can be calculated as $\sum_{t = 1}^N |\widehat{\mathcal{A}}_g(s_t)| = \sum_{t = 1}^N \sum_{v \in \mathcal{V}_{s_t}} |\mathcal{N}(v) \setminus \mathcal{V}_{s_t}|$. Although the value of $|\mathcal{N}(v) \setminus \mathcal{V}_{s_t}|$ varies for different nodes $v$ and at different steps $t$, we treat them as having the same order of magnitude and denote the average value as $D$. Consequently, the complexity of processed candidates in the subgraph-generation-based method can be expressed as $O\left(\frac{N^2 + N}{2}D\right)$.
For the walk-exploration-based method, we calculate the number of processed candidates as $\sum_{t = 1}^N |\widehat{\mathcal{A}}_w(s_t)| = \sum_{t = 1}^L |\mathcal{N}(v_t)|$. Assuming $|\mathcal{N}(v_t)|$ is of the same order of magnitude as $D$, the complexity of the walk-exploration-based method becomes $O(LD)$. Consequently, our proposed walk-exploration-based method reduces the complexity of the subgraph-generation-based method from quadratic to linear complexity.}

\section{Numerical Evaluation}
In this section, we present numerical evaluations of our proposed method. Our objectives are threefold: (A) to assess the performance of our method in downstream tasks, (B) to investigate its explainability, and (C) to analyze the necessity of our improvement of the walk-based MDP. We report the experimental results and provide detailed analysis.

\subsection{Experimental setup}
\paratitle{Datasets and criterion}. We conduct our experiments in seven graph classification tasks on real-world graph datasets: MUTAG \cite{debnath1991structure}, PTC-MR \cite{helma2001predictive}, BZR, COX2\cite{sutherland2003spline}, PROTEINS \cite{borgwardt2005protein}, IMDB-BINARY \cite{yanardag2015deep} and NCI109 \cite{wale2008comparison}. We download the source data from the TUdataset \cite{Morris+2020} implemented with Pytorch geometric \cite{fey2019fast}. All the datasets are for the binary graph classification tasks, where each node in the graphs are assigned with a discrete label. We perform inference using a 10-fold cross-validation on each dataset. Initially, the entire dataset is split into 10 folds, with 9 of them serving as the training subset and 1 as the validation subset. Similar to many related works \cite{wijesinghe2021new, sun2021sugar}, we use the average and standard deviation of the highest validation accuracies as our evaluation metric. In the training process for graph classification, we use the cross-entropic loss to define the loss function $l$ in Equation (\ref{Eq.Sop}).

\paratitle{Comparison methods}. We compare our proposed methods with other methods of two categories: (i) graph kernels and (ii) graph neural networks. On the one hand, category (i) contains WL \cite{Shervashidze_JMLR_2011_s}, WWL \cite{togninalli2019wasserstein}, and FGW \cite{titouan2019optimal}, which are effective and classic Weisfeiler-Lehman-based graph kernels. On the other hand, category (ii) contains GIN \cite{xu2018powerful}, DGCNN \cite{maron2019provably}, GraphSAGE\cite{hamilton2017inductive}, SUGAR \cite{sun2021sugar}, DSS-GNN \cite{bevilacqua2021equivariant} and GraphSNN \cite{wijesinghe2021new}. It is noteworthy that SUGAR is also based on the subgraph extraction with reinforcement learning, where they additionally adopt a self-supervised mutual information (MI) loss. Therefore, we only report their performance without MI loss to keep the training process the same as others. 

\paratitle{Parameter settings}. We directly take the results of comparison methods from their papers. Therefore, we only adjust the hyper-parameters of our proposed methods with the following parameter settings: (i) shallow MPNN: we adopt a $3$-layer GIN \cite{xu2018powerful} as the shallow MPNN for our low-level graphlet perceptron and graphlet encoder. (ii) MDP settings: for the setting of the maximum trajectory length $L$, we adjust it within $\{8, 16\}$ for both the subgraph-based MDP and walk-based MDP. We also adjust the number of substructure samples extracted in a single graph within $\{3, 16, 32\}$. (iii) $Q$-learning settings: we set the $Q$-learning parameters as $\gamma = 0.9$, $\beta = 0.1$. The DQN is trained with $\epsilon$ gradually increasing from $0.1$ to $0.4$ as the epochs progress.

\begin{table*}
\caption{Average and standard derivation results of graph datasets. ``-'' represents that the result is not reported in the original paper.}
\label{Tab.acc}
\begin{center}
\scalebox{0.8}{
\begin{tabular}{lccccccc}
\toprule
Methods&MUTAG&PTC-MR&BZR&COX2&PROTEINS&IMDB-B&NCI109\\
\midrule
WL&90.40$\pm$5.7&59.90$\pm$4.30&78.5$\pm$0.60&81.7$\pm$0.70&75.00$\pm$3.10&73.80$\pm$3.90&82.46\cr
WWL&87.20$\pm$1.50&66.30$\pm$1.20&84.40$\pm$2.00&78.20$\pm$0.40&74.20$\pm$0.50&74.30$\pm$0.80&-\cr
FGW&88.40$\pm$5.60&65.30$\pm$1.20&85.10$\pm$4.10&77.20$\pm$4.80&74.50$\pm$2.70&63.80$\pm$3.40&-\cr
\midrule
GIN&89.40$\pm$5.60&64.60$\pm$7.00&-&-&75.90$\pm$2.80&75.10$\pm$5.10&-\cr
DGCNN&85.80$\pm$1.70&58.60$\pm$2.50&-&-&75.50$\pm$0.90&70.00$\pm$0.90&-\cr
GraphSAGE&85.10$\pm$7.60&63.90$\pm$7.70&-&-&75.90$\pm$3.20&72.30$\pm$5.30&-\cr
SUGAR (NoMI)&91.11&70.38&-&-&74.07&-&80.57\cr
DSS-GNN&91.10$\pm$7.00&69.20$\pm$6.50&-&-&75.90$\pm$4.30&77.10$\pm$3.00&{\bf 82.80$\pm$1.20}\cr
GraphSNN&{\bf 94.40$\pm$1.20}&71.01$\pm$3.60&{\bf 91.88$\pm$3.20}&{\bf 86.72$\pm$2.90}&{\bf 78.21$\pm$2.90}&{\bf 77.87$\pm$3.10}&-\cr
\midrule
SubgraphMDP&92.57$\pm$5.4&{\bf 75.60$\pm$4.26}&90.37$\pm$3.57&84.59$\pm$2.45&73.76$\pm$4.29&75.20$\pm$2.32&71.99$\pm$1.97\cr
WalkMDP&{\bf 97.31$\pm$4.35}&{\bf 75.55$\pm$6.87}&{\bf 93.34$\pm$3.13}&{\bf 88.65$\pm$2.69}&{\bf 76.81$\pm$3.79}&{\bf 77.30$\pm$1.95}&{\bf 83.93$\pm$1.31}\cr
\bottomrule
\end{tabular}}
\end{center}
\end{table*}

\begin{figure*}
	\centering
	\includegraphics[width= \hsize]{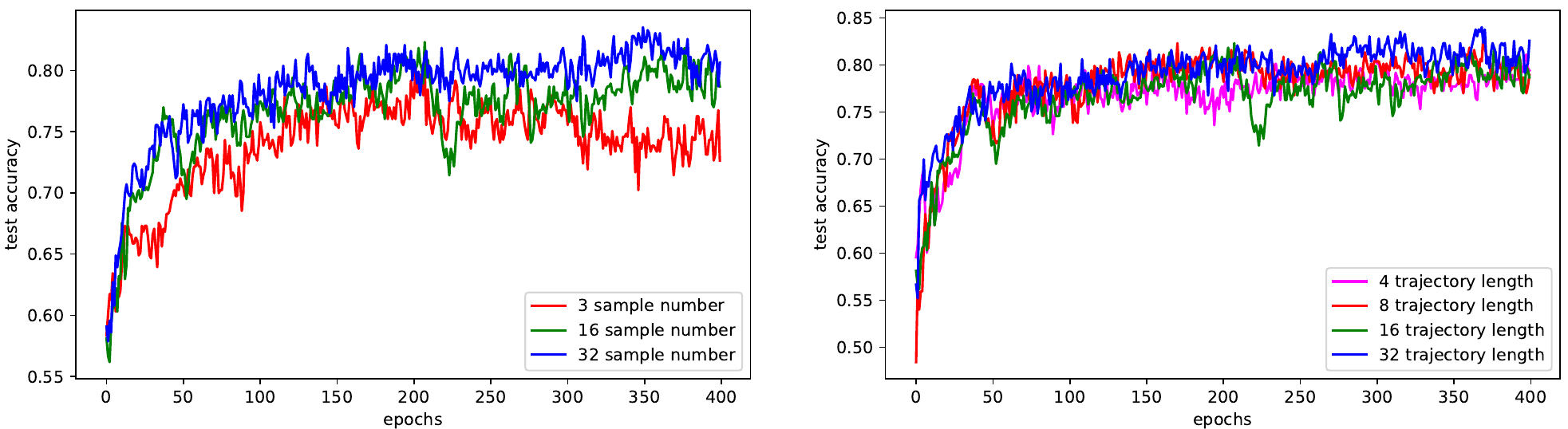}

	\caption{Test accuracy of walk-based MDP for varying trajectory lengths and varying sample numbers on the NCI109 dataset. The x-axis represents the number of epochs, while the y-axis corresponds to the classification accuracy (with a maximum value of $1$ denoting $100\%$).}
	\label{Fig.abla1}
\end{figure*}

\subsection{Experimental results and analysis}
\paratitle{Graph classification performance}. The results of graph classification tasks are reported in Table \ref{Tab.acc}, where the top-$2$ results are marked in a bold typeface. Compared to other methods, our proposed methods consistently achieves top-2 performance across all datasets. Notably, it attains the top position on the MUTAG, PTC-MR, BZR, and COX2 datasets. These experiments produce expected results responding to our experimental objective (A), where our proposed methods show their effectiveness on graph classification tasks. As for our experimental objective (C), by comparing the performance across the rows of Subgraph-based MDP (SubgraphMDP) and Walk-based MDP (WalkMDP), we observe that our proposed reinforcement substructure generation framework significantly enhances the performance compared to the subgraph-based MDP.

\paratitle{Subgraph visualization}. To achieve our experimental objective (B), we conduct the subgraph visualization experiments on a real-world dataset: MUTAG, and a synthetic dataset: BA-2motifs \cite{luo2020parameterized}. According to the prior knowlege of biochemical engineering \cite{lin2022orphicx}, the chemical group $\rm NO_2$ plays a very important role in classifying mutagenic and non-mutagenic molecules. However, there are exceptions where $\rm NO_2$ alone cannot directly determine mutagenicity. As for the BA-2motifs dataset, it contains $1000$ Barabasi-Albert (BA) graphs with two categories of substructure patterns/motifs: circle and house. We conduct the graph classification tasks on these two datasets and visualize the extracted important substructures using our proposed method with walk-based MDP. {In these experiments, we only sample one walk sequence with maximum length as $16$}. Figure \ref{Fig.vis1} and Figure \ref{Fig.vis2} show the visualization results on MUTAG and BA-2motifs, respectively. In MUTAG dataset, we present two samples of mutagenic molecules (above the epoch arrow), and two samples of non-mutagenic molecules (below the epoch arrow). Figure \ref{Fig.vis1} shows the evolving process of extracted subgraphs as epochs progress. Our proposed method not only correctly extracts the $\rm NO_2$ chemical group but also effectively avoids the exception cases where $\rm NO_2$ is present but the graph label differs. In BA-2motif dataset, we also present two samples of circle motif and two samples of house motif. Similarly, the extracted substructures in Figure \ref{Fig.vis2} also shows that our proposed method is able to provide convincing explainability.


\begin{figure}
	\centering
	\includegraphics[width= 0.8\hsize]{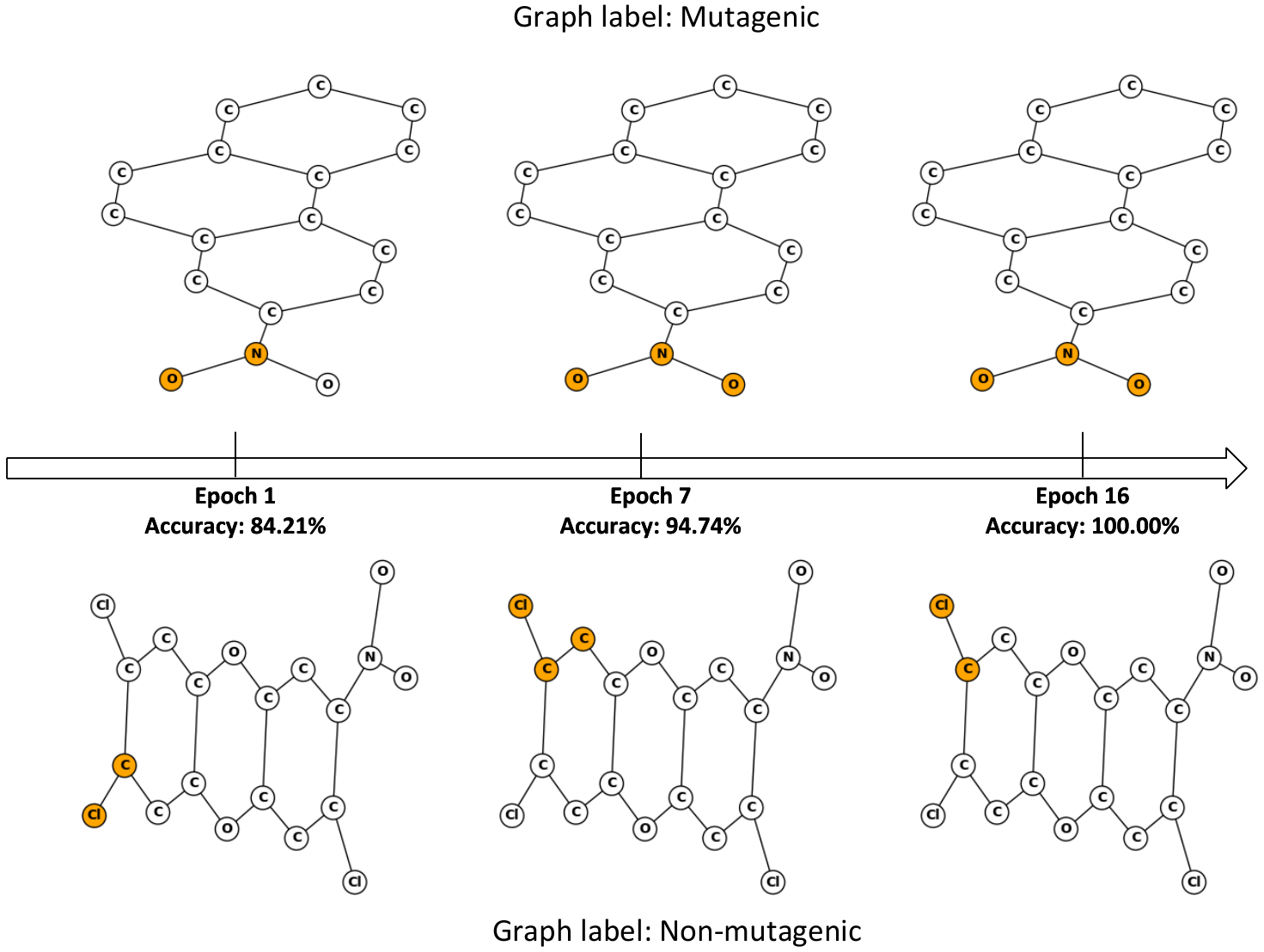}
	\includegraphics[width= 0.8\hsize]{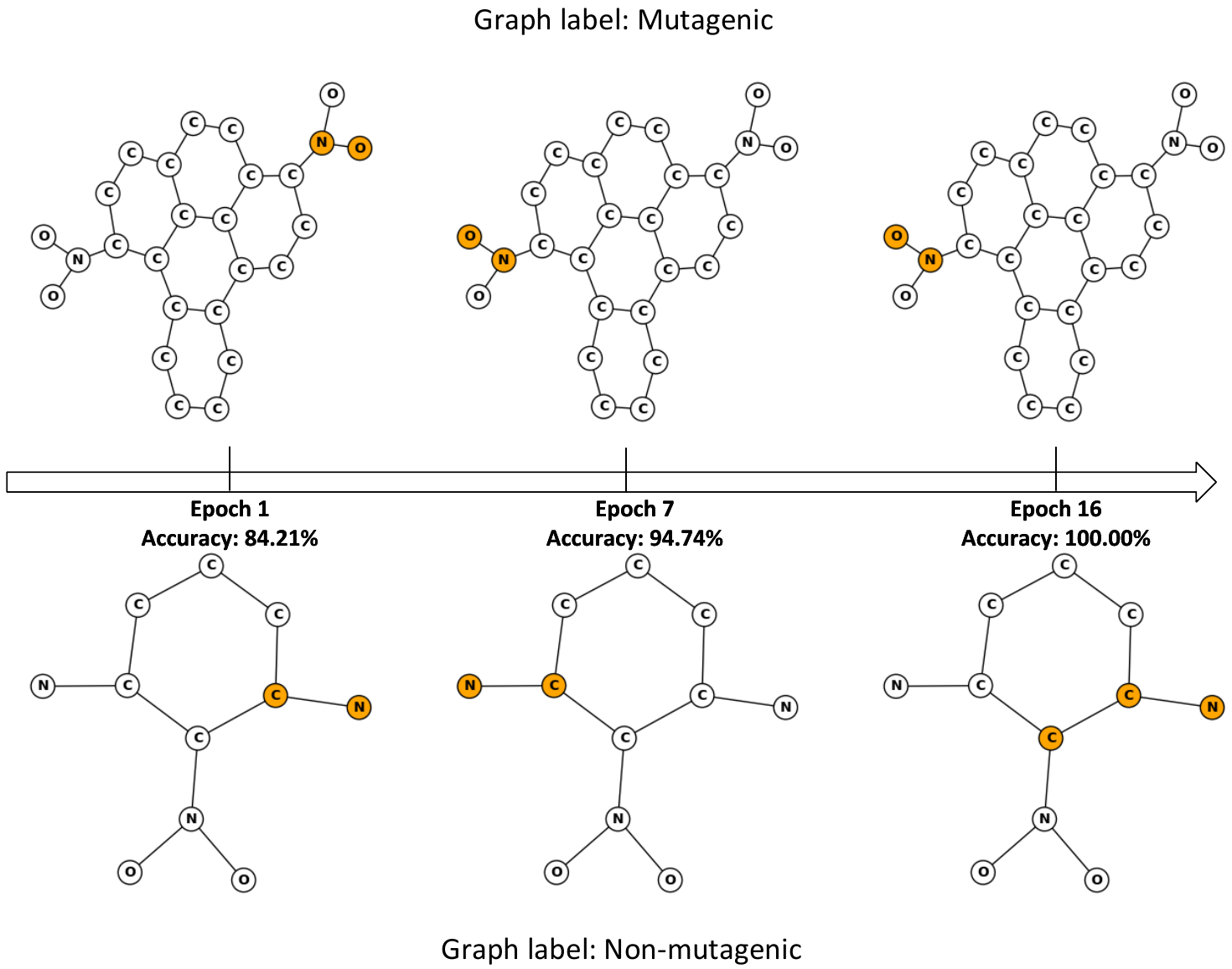}

	\caption{The extracted subgraphs at different epochs in MUTAG dataset. The nodes in orange color denote the extracted nodes.}
	\label{Fig.vis1}
\end{figure}

\begin{figure}
	\centering
	\includegraphics[width= 0.8\hsize]{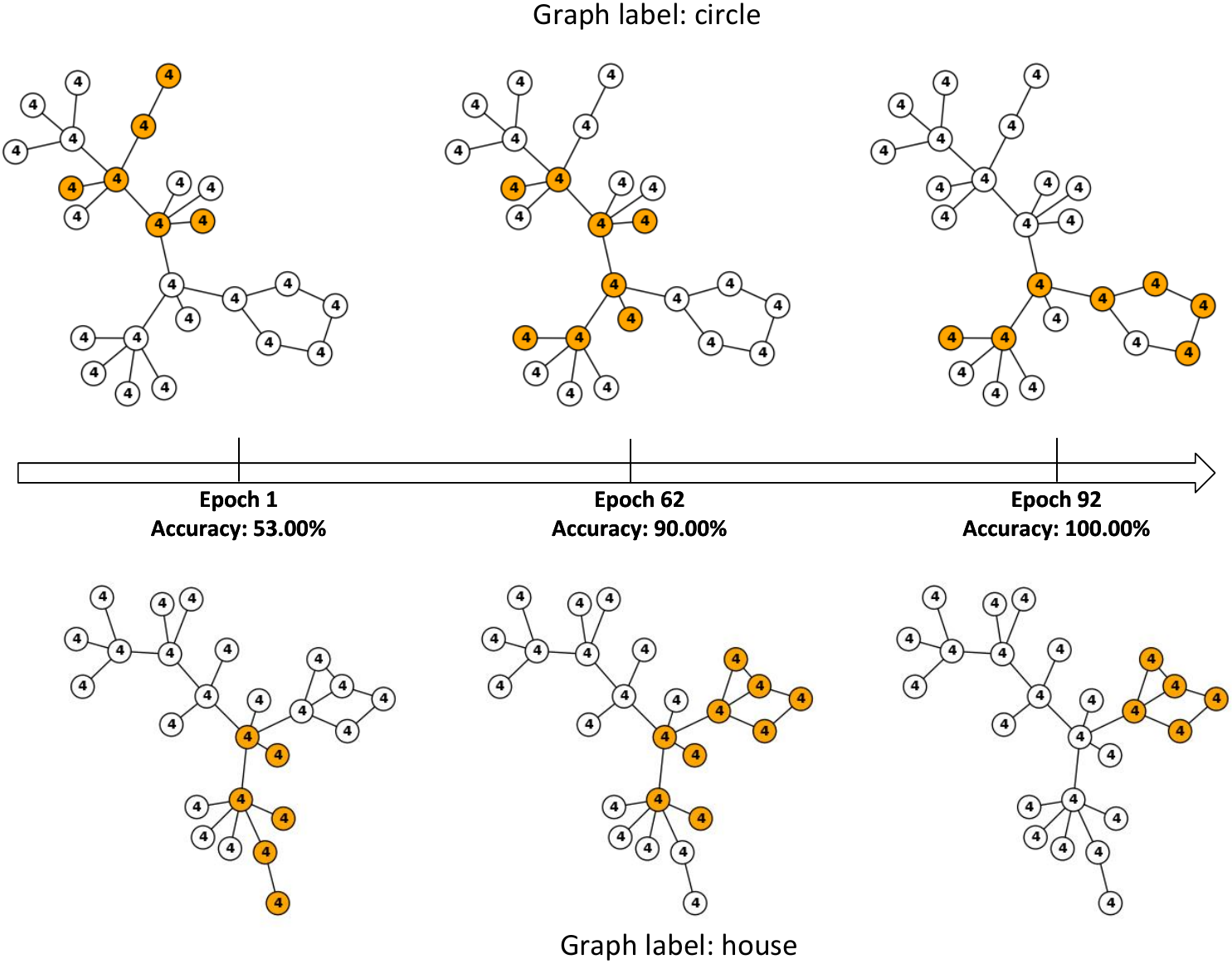}
	\includegraphics[width= 0.8\hsize]{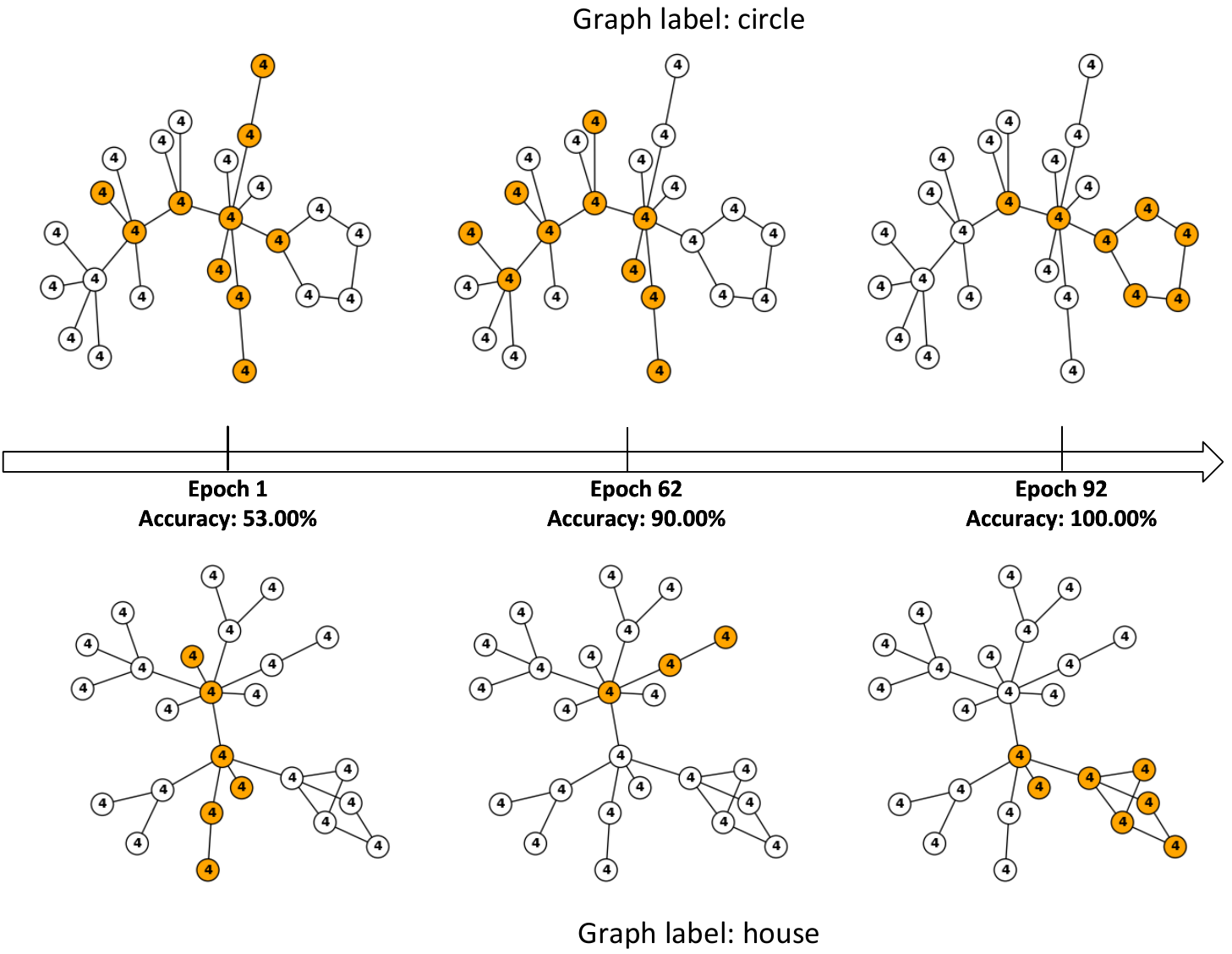}

	\caption{The extracted subgraphs at different epochs in BA-2motifs dataset. The nodes in orange color denote the extracted nodes.}
	\label{Fig.vis2}
\end{figure}

\paratitle{Ablation studies}. To explore the impact of hyperparameters in our proposed method, we perform ablation studies across various hyperparameter configurations. Specifically, we focus on two major hyperparameters: trajectory length and sample number. The trajectory length $L$ signifies the maximum length of the generation trajectory, while the sample number represents the number of substructures simultaneously extracted for each graph. Thereafter, the encoded multiple extracted substructures are aggregated using the pooling strategy denoted by $f_{\rm PL}$ in Equation (\ref{Eq.output}). We conduct experiments with our proposed walk-based MDP on NCI109 dataset, where trajctory lengths are within $\{4,8,16,32\}$, the sample numbers are within $\{3,16,32\}$. Figure \ref{Fig.vis2} shows the test accuracy curves of different hyperparameters. The curves in the right picture represents the results obtained for varying trajectory lengths, while keeping the sample number fixed at 16. Conversely, the left picture illustrates the results obtained for different sample numbers, with the trajectory length held constant at 16. We observe that increasing the sample number can produce a more stable accuracy curve during training, with sacrificing the computational cost. {Furthermore, the curves corresponding to different trajectory lengths demonstrate that our proposed method is not significantly affected by the length of trajectories. This property is particularly advantageous when we lack information about the sizes of important substructures.} 

\section{Conclusion}
In this paper, we introduce a novel self-supervised SGNN leveraging deep reinforcement walk exploration to address the inherent challenges of {inefficient sampling and lack of explainability}. Our method combines {the principle of SGNNs with the generation approach of GNN explainers to achieve the self-supervised framework}. Our proposed methods detect relevant graphlets and capture complex subgraphs through a reinforcement learning-based walk exploration process. The proposed method not only improves the explainability of the model but also significantly enhances computational efficiency. Future work will focus on exploring the integration of our RWE-SGNN with other advanced GNN architectures and extending its applicability to dynamic and heterogeneous graphs. Additionally, investigating the model performance in the context of multi-agent system will provide deeper insights and potentially lead to further enhancements in performance and generalization capabilities.

\newpage
\appendix
\section{Pseudo-codes}
This appendix contains the pseudo-codes of both the training and the inference process of our proposed model.
\begin{algorithm}[H]
\caption{Training process of sampling model and output model}
\label{Alg.train}
\begin{algorithmic}[1]
\State Initialize sampling model parameters $\Theta, \Phi, \Phi^\prime$ and output model parameters $\Theta^\prime, \Theta^{\prime\prime}$
\For{Each epoch $e$}
	\State Sample a graph data batch of size $B$ as $\{(G^{(b)}, \mat{X}^{(b)})\}_{b=1}^B$
	\State // Training sampling model
	\State Initialize the total $Q$-learning loss $l_Q \gets 0$
	\For{Each graph data $(G^{(b)}, \mat{X}^{(b)})$}
		\State Compute the low-level graphlet-aware embeddings $\mat{Z}^{(b)} \leftarrow {\rm MPNN}_{\Theta}^{(p,d,k)}(G^{(b)}, \mat{X}^{(b)})$
		\State Define the initial state $s_0^{(b)}$
		\For{Each generation iteration $t$}
			\State Generate action $a\gets \pi_\epsilon(s_t^{(b)}, q)$ with $q$ sampled from a uniform distribution $U(0,1)$
			\State Compute the $Q$-learning loss $l_Q \gets l_Q + \left|Q^p_{\Phi}(E(s, \mat{Z}), a) - Q_E(s, a) \right|$
			\State Update the next state $s_{t+1}^{(b)}$ with the generated action $a$
		\EndFor
	\EndFor
	\State Update parameter $\Theta,\Phi \gets \mathop{\rm arg~min}_{\Theta, \Phi} l_Q$
	\State Update parameter $\Phi^\prime \gets \beta\cdot\Phi + (1-\beta)\cdot\Phi^\prime$
	\State // Training output model
	\State Initialize the total output loss $l_O \gets 0$
	\For{Each graph data $(G^{(b)}, \mat{X}^{(b)})$}
		\State Compute the node embedding $\widehat{\mat{Z}}^{(b)} \gets{\rm MPNN}_{\Theta^\prime}^{(p,d,k)}(G^{(b)}, \mat{X}^{(b)})$
		\State Define the initial state $s_0^{(b)}$
		\For{Each generation iteration $t$}
			\State Generate action $a\gets \pi_\epsilon(s_t^{(b)}, 1)$
			\State Update the next state $s_{t+1}^{(b)}$ with the generated action $a$
		\EndFor
		\State Obtain the final state after $T$ iterations as $s_T^{(b)}$
		\State Compute the downstream output loss $l_O \gets l_O +  l({\rm MLP}_{\Theta^{\prime\prime}}(f_{\rm PL}(E(s_T^{(b)}, \widehat{\mat{Z}}))))$
	\EndFor
	\State Update parameter $\Theta^\prime, \Theta^{\prime\prime} \gets \mathop{\rm arg~min}_{\Theta^\prime, \Theta^{\prime\prime}}l_Q$
\EndFor
\end{algorithmic}
\end{algorithm}
\begin{algorithm}[H]
\caption{Inference process of sampling model and output model}
\label{Alg.test}
\begin{algorithmic}[1]
\State Initialize sampling model parameters $\Theta, \Phi, \Phi^\prime$ and output model parameters $\Theta^\prime, \Theta^{\prime\prime}$

	\State Compute the low-level graphlet-aware embeddings $\mat{Z} \leftarrow {\rm MPNN}_{\Theta}^{(p,d,k)}(G, \mat{X})$
	\State Define the initial state $s_0^{(b)}$
	\For{Each generation iteration $t$}
		\State Generate action $a\gets \pi_\epsilon(s_t^{(b)}, 1)$
		\State Compute the $Q$-learning loss $l_q \gets \left|Q^p_{\Phi}(E(s, \mat{Z}), a) - Q_E(s, a) \right|$
		\State $l_Q \gets l_Q + l_q$
		\State Update the next state $s_{t+1}^{(b)}$ with the generated action $a$
	\EndFor
	\State Obtain the final state after $T$ iterations as $s_T$
	\State Compute the graph embedding $\vec{o}\gets{\rm MLP}_{\Theta^{\prime\prime}}(f_{\rm PL}(E(s_T, \mat{Z}^\prime)))$
\end{algorithmic}
\end{algorithm}

\newpage
\bibliographystyle{plain}
\bibliography{nips, graph, OT, graph_datasets}
	
\end{document}